\documentclass[letterpaper, 10 pt, journal, twoside]{IEEEtran} \IEEEoverridecommandlockouts
\usepackage{cite}
\usepackage{amsmath,amssymb,amsfonts}
\usepackage{algorithmic}
\usepackage{graphicx}
\usepackage{float}
\usepackage{textcomp}
\usepackage{xcolor}
\usepackage{mymacros}
\usepackage{afterpage}
\usepackage{amsthm}
\usepackage{blindtext}
\usepackage{hyperref}

\newtheorem{theorem}{Theorem}[section]
\newtheorem{lemma}[theorem]{Lemma}

\begin{document}

\title{Learning Safe, Generalizable Perception-based Hybrid Control with Certificates}

\author{Charles Dawson$^{1}$, Bethany Lowenkamp$^{2}$, Dylan Goff$^1$, and Chuchu Fan$^{1}$%
\thanks{Manuscript received: September 9, 2021; Revised November 23, 2021; Accepted December 30, 2021.}
\thanks{This paper was recommended for publication by Editor Jens Kober upon evaluation of the Associate Editor and Reviewers' comments.
This work was supported by the NASA University Leadership Initiative (grant \#80NSSC20M0163) and the Defense Science and Technology Agency in Singapore, but this article solely reflects the opinions and conclusions of its authors and not any NASA entity, DSTA Singapore, or the Singapore Government. CD is supported by the NSF GRFP under Grant No. 1745302.} 
\thanks{$^{1} $CD, DG, and CF are with the Dept. of Aeronautics and Astronautics, MIT, Cambridge MA
        {\tt\footnotesize \{cbd, dgoff, chuchu\}@mit.edu}}%
\thanks{$^{2} $BL is with the Dept. of Mechanical Engineering, MIT, Cambridge MA
        {\tt\footnotesize bethlow@mit.edu}}%
\thanks{$^{*} $BL and DG contributed equally.}%
\thanks{Digital Object Identifier (DOI): see top of this page.}
}


\maketitle

\begin{abstract}
Many robotic tasks require high-dimensional sensors such as cameras and Lidar to navigate complex environments, but developing certifiably safe feedback controllers around these sensors remains a challenging open problem, particularly when learning is involved. Previous works have proved the safety of perception-feedback controllers by separating the perception and control subsystems and making strong assumptions on the abilities of the perception subsystem. In this work, we introduce a novel learning-enabled perception-feedback hybrid controller, where we use Control Barrier Functions (CBFs) and Control Lyapunov Functions (CLFs) to show the safety and liveness of a full-stack perception-feedback controller. We use neural networks to learn a CBF and CLF for the full-stack system directly in the observation space of the robot, without the need to assume a separate perception-based state estimator. Our hybrid controller, called LOCUS (Learning-enabled Observation-feedback Control Using Switching), can safely navigate unknown environments, consistently reach its goal, and generalizes safely to environments outside of the training dataset. We demonstrate LOCUS in experiments both in simulation and in hardware, where it successfully navigates a changing environment using feedback from a Lidar sensor.
\end{abstract}


\section{Introduction}

\IEEEPARstart{U}{sing} visual input to control autonomous systems without compromising safety or soundness is a challenging problem in robotics. Traditional methods from control theory provide powerful tools for safety and performance analysis but lack the expressiveness to deal with rich sensing models such as vision or Lidar. On the other hand, learning-based methods have been used successfully on visual-feedback control tasks including autonomous driving~\cite{Pan2017AgileOA} and aircraft taxiing~\cite{taxinet}, but ensuring the safety of these controllers remains an open question.

Most existing work analyzing the safety of learned vision-based controllers has been done \textit{post hoc}, where the controller is synthesized and then independently verified. Works such as~\cite{julian2020} assess safety through adversarial testing, searching for an input sequence that causes a learned vision-based controller to fail. Others such as~\cite{Katz2021dasc} learn a generative model to predict observations from states, then conduct a reachability analysis to check the safety of the concatenated generator-controller network. In both cases, controller safety and controller synthesis are treated as two separate issues, rather than using safety considerations to inform the synthesis process.

\begin{figure}[t!]
    \centering
    \includegraphics[width=\linewidth]{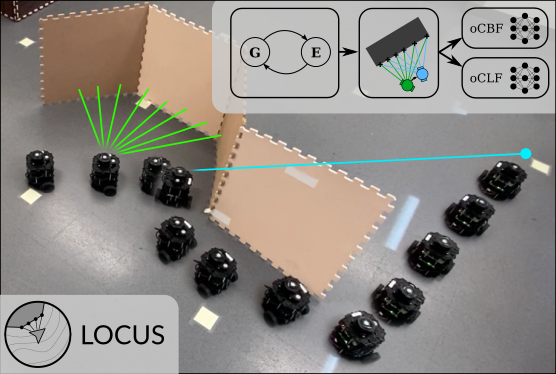}
    \caption{Our controller combines a learned Control Barrier Function (based on Lidar sensor data, green) and a learned Control Lyapunov Function (based on range and bearing relative to the goal, blue) with a hybrid control architecture. The robot switches between goal-seeking and exploratory modes, but is guaranteed to maintain safety (using the barrier function) and eventually reach the goal if a safe path exists (using the Lyapunov function).}
    \label{fig:headline}
\end{figure}

Integrating safety into the control synthesis process has born impressive results in the state-feedback regime, in the form of certificate learning for control. These approaches draw on control theoretic certificates such as Lyapunov functions~\cite{Giesl2015,Dai2020}, barrier functions~\cite{ames_cbf,Peruffo2020}, and contraction metrics~\cite{Sun2020,Singh2020} that prove properties such as stability, forward invariance, and differential tracking stability of dynamical systems. Recent works have demonstrated that learning these certificates alongside a state-feedback control policy can enable provably safe controllers for walking~\cite{Castaneda2020}, stable flight in turbulent conditions~\cite{Sun2020}, and decentralized multi-agent control~\cite{Qin2021}.

Several works have attempted to extend these approaches to observation-feedback control synthesis~\cite{Dean2020,Dean2020a,anonymous2021robust}, but these works make a number of \textcolor{black}{strong} assumptions. For instance, \cite{Dean2020a} and~\cite{anonymous2021robust} assume that the robot's state can be inferred with bounded error from a single observation, limiting them to ``third-person'' camera perspectives where each observation completely captures the state of the robot and its environment. As a result, these approaches do not readily apply when the vision system (either a camera or Lidar scanner) is mounted to the robot and cannot observe the entire environment. Further, these approaches assume that the environment (including the locations of any obstacles) is known prior to deployment, limiting the ability of these controllers to generalize to new environments. \textcolor{black}{We relax these assumptions to permit deployment on a robot in previously-unseen environments, although we retain some assumptions on the sensor model (see~\ref{assumptions})}.

In this work, our main contribution is to combine techniques from machine learning and hybrid control theory to learn safe observation-feedback controllers that generalize to previously-unseen environments. In particular, we combine learned observation-based barrier and Lyapunov functions with a novel certificate-based hybrid control architecture. We provide proofs that our hybrid controller will remain safe and reach its goal without deadlock when provided with valid certificate functions, \textcolor{black}{and we use neural networks to learn certificates that are valid with high probability.}

To our knowledge, this represents the first attempt to provide safety and goal-reaching guarantees for a perception-feedback controller that considers the full autonomous navigation stack (rather than sand-boxing the perception system with assumptions on error rates and considering the controller separately). We validate our approach, called LOCUS (Learning-enabled Observation-feedback Control Using Switching) both in simulation and in hardware, where we demonstrate a mobile robot navigating an unseen environment with Lidar measurements using our learned hybrid controller. \textcolor{black}{Experiments show that our learned controller can be run 6 times faster than MPC, reaches the goal almost twice as quickly as end-to-end reinforcement learning policies, and maintains a perfect safety rate in previously-unseen environments.}

\section{Problem Statement}

Our goal is to synthesize an observation-feedback controller for a robot with nonlinear dynamics navigating an unknown environment. The robot perceives its environment using a Lidar sensor, as in Fig.~\ref{fig:headline}, and it must navigate to a goal point while avoiding collisions. In contrast with higher-level approaches that combine SLAM with a global path planner, or robust planning approaches like~\cite{ChenFaSTrack}, we restrict our focus to real-time controllers with feedback from local observations. Our observation-feedback controller can be combined with SLAM and planning modules to track waypoints along a path, but it can also be used without those components, or when those components fail, without compromising safety.

Formally, we consider a robot with nonlinear discrete-time dynamics $x_{t+1} = f(x_t, u_t)$, where $x \in \mathcal{X} \subseteq \R^n$ and $u \in \mathcal{U} \subseteq \R^m$, and observation model $o_t = o(x_t) \in \mathcal{O} \subseteq \R^p$ (denote by $o_t^i$ the $i$-th element of this observation). In particular, we focus on a Lidar observation model that measures the locations where $n_{rays}$ evenly-spaced rays originating at the robot first make contact with an object in the environment. In this model, $\mathcal{O} = \R^{2n_{rays}}$, and the points are expressed in robot-centric coordinates. Given a goal location $x_g$, we seek a control policy $u_t = \pi(o_t)$ that drives $x_t \to x_g$ as $t\to\infty$. Moreover, we require that $u_t$ avoids colliding with obstacles in the environment, defined as occurring when $\min_{i=0,\ldots,n_{rays}} ||o_t^i|| \leq d_c$, where $d_c > 0$ is a user-defined margin (i.e. collision occurs when the robot gets too close to any obstacle in the environment).

We make a number of assumptions about the robot's dynamics and capabilities, as well as on the structure of the environment. Some of these assumptions are standard, while others are particular to our approach, as outlined in Section~\ref{overview}.

\subsection{Assumptions}\label{assumptions}

\subsubsection{Dynamics \& State Estimation}\label{assumptions_dynamics}

Our approach is applicable in two regimes. If a (potentially noisy) state estimate is available, we make no assumptions on the dynamics other than controllability (a standard assumption). If no state estimate is available (other than a measurement of range $\rho$ and bearing $\phi$ to the goal point), then we assume that the robot's dynamics are controllable and \textit{approximately local}, by which we mean the state update has the form $x_{t+1} - x_{t} \approx f_\Delta(u_t)$. This means that the change in state (expressed in a local frame) can be approximated using only the control input (i.e. without a state estimate).
This second case is motivated by our choice of hardware platform (a mobile robot with Dubins car dynamics, which satisfy this assumption and thus do not require any state estimation), but we include both cases for completeness.

\subsubsection{Observation Model}

We assume that the robot observes its environment using a 2D, finite-resolution, finite-range Lidar sensor. This sensor measures the $xy$ location in the robot frame of the first point of contact with an obstacle along each of $n_{rays}$ evenly-spaced rays. If no contact point is detected within maximum range $d_o$, the measurement saturates at that range. This model distinguishes us from less-realistic overhead camera models in \cite{anonymous2021robust}, \textcolor{black}{but it is important to note that we do not explicitly model sensor noise (we hope to extend our theoretical analysis to cover sensor noise in a future work).}

\subsubsection{Environment}

We assume only that there are no ``hedgehog'' obstacles with spikes thinner than the gap between adjacent Lidar beams and that all obstacles are bounded.

\section{Overview}\label{overview}

To solve this safe navigation problem, we propose the controller architecture in Fig.~\ref{fig:architecture}. The core of this controller is a Control Barrier Function (CBF) defined on the space of observations, which enables safe observation-based control by guaranteeing that the robot will not collide with any obstacles. Unlike other CBF-based approaches, which define the CBF in terms of the state-space and require accurate state estimates and knowledge of the environment \cite{ames_cbf}, our approach defines the CBF directly in observation-space. We also use a Control Lyapunov Function (CLF) to guide the robot towards its goal.

In an offline phase, we use neural networks to learn the observation-space CBF and CLF. In an online phase, we use the learned CBF and CLF in a hybrid controller that switches between goal-seeking and exploratory behavior to safely navigate an \textit{a priori} unknown environment. To ensure that the CBF and CLF constraints are respected during the online phase, the controller uses its knowledge of system dynamics to approximate future observations with a one-step horizon, using this ``approximate lookahead'' to select optimal control actions given the CBF/CLF constraints. Section~\ref{certificates} describes an extension of CBF theory to include control from observations, and Section~\ref{hybrid_controller} describes the hybrid goal-seeking/exploratory controller, including the approximate lookahead strategy. Section~\ref{learning} describes our approach to learning observation-space CBFs and CLFs using neural networks. Section~\ref{experiments} presents experiments, both in simulation and in hardware, demonstrating the performance of our approach.

\begin{figure}[b]
    \centering
    \includegraphics[width=0.8\linewidth]{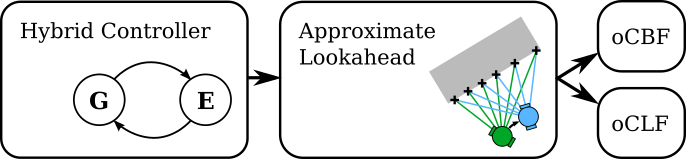}
    \caption{Block diagram of controller architecture.}
    \label{fig:architecture}
\end{figure}

\section{Ensuring Safety and Convergence with Observation-space Certificates}\label{certificates}

As shown in Fig.~\ref{fig:architecture}, at the foundation of our approach is a pair of certificate functions: a Control Barrier Function (CBF) that allows the robot to detect and avoid obstacles, and a Control Lyapunov Function (CLF) that guides the robot towards its goal. Existing approaches to designing CBF/CLF-based controllers have defined these certificates as functions of the robot's state \cite{ames_cbf,Choi2020,Qin2021,Dean2020a}; however, this dependence on state means that state-based certificates, particularly CBFs, have difficulty generalizing to new environments. For example, a state-based CBF might encode regions of the state space corresponding to obstacles, but if the obstacles move then the entire CBF is invalidated.

A more flexible implementation is to define these certificates as functions of observations (e.g. Lidar returns and the range and bearing to the goal). Instead of encoding which regions of the state space are ``inside'' obstacles, an observation-based CBF can encode which observations indicate that the robot is unsafe (e.g. when the minimum distance measured via Lidar is less than some threshold). In this paper, we extend traditional state-based CBFs to observation-based CBFs (oCBFs) defined as functions of observations . Formally, oCBFs are scalar functions $h: \mathcal{O} \mapsto \R$ such that for some $0 \leq \alpha_h < 1$
\begin{align}
    o\in\mathcal{O}_{safe} \ \implies \quad &h(o) \leq 0 \label{ocbf_safe}\\
    o\in\mathcal{O}_{unsafe} \ \implies \quad &h(o) \geq 0 \label{ocbf_unsafe} \\
    \forall x\in\mathcal{X}\ \exists u\in\mathcal{U}\ \text{s.t.}\quad &h\left[o(f(x_t, u))\right] - \alpha_h h[o(x_t)] \leq 0 \label{ocbf_dynamics}
\end{align}
Similarly, an observation-based CLF (oCLF) can be defined as a scalar function $V: \mathcal{O} \mapsto \R$ such that for some $0 \leq \alpha_V < 1$
\begin{align}
    V(o) \geq 0;\quad o\in\mathcal{O}_{goal}\iff V(o) = 0  \label{oclf_pd}\\
    \forall x\in\mathcal{X}\ \exists u\in\mathcal{U}\ \text{s.t.}\  V\left[o(f(x_t, u))\right] - \alpha_V V[o(x_t)] \leq 0 \label{oclf_dynamics}
\end{align}
Of course, these definitions require that the space of observations $\mathcal{O}$ is rich enough to both differentiate between safe and unsafe sets $\mathcal{O}_{safe}$ and $\mathcal{O}_{unsafe}$ (e.g. via Lidar measurements) and detect when the robot has reached the goal in $\mathcal{O}_{goal}$ (e.g. by measuring range to the goal). In some ways, this extension is trivial, since every oCBF $h: \mathcal{O} \mapsto \R$ (resp. oCLF $V: \mathcal{O} \mapsto \R$) defines a state-space CBF $h \circ o: \mathcal{X} \mapsto \R$ (resp. CLF $V \circ o: \mathcal{X} \mapsto \R$), and so oCBFs and oCLFs inherit the safety and convergence guarantees of their state-space equivalents. In particular, as long as a policy selects actions from the set $\mathcal{K}_{CBF} = \set{u\ :\ h\left[o(f(x, u))\right] - \alpha_h h[o(x)] \leq 0}$, the system will remain safe (or, if started in the unsafe region, move to the safe region). Similarly, as long as a policy selects actions from the set $\mathcal{K}_{CLF} = \set{u\ :\ V\left[o(f(x, u))\right] - \alpha_h V[o(x)] \leq 0}$, the system will converge to the goal.

In addition to inheriting these safety and convergence guarantees, oCBFs and oCLFs also inherit the two main drawbacks of CBFs and CLFs. First, it is very difficult to construct $h$ and $V$ by hand to satisfy conditions~\eqref{ocbf_dynamics} and~\eqref{oclf_dynamics}. Second, if the intersection $\mathcal{K}_{CBF} \cap \mathcal{K}_{CLF}$ empty, then traditional CBF/CLF controllers have to choose between preserving safety or convergence. As such, the novelty of our approach is not in the straightforward extension to observation-based certificates, but in our solutions to these two drawbacks. To address the first drawback, we use neural networks to learn an oCBF and oCLF, allowing our approach to generalize to previously-unseen environments; to solve the second issue, we use a novel hybrid control architecture to prove the safety and convergence of the resulting learning-enabled controller. We discuss each of these advances in the next two sections; we begin by discussing the hybrid control architecture, then describe how to learn oCBFs and oCLFs for use with that controller.

\section{Hybrid CBF/CLF Control Policy}\label{hybrid_controller}

oCBF and oCLF certificates guarantee safety and stability by defining the sets of allowable control inputs $\mathcal{K}_{CBF}$ and $\mathcal{K}_{CLF}$. As long as the input lies within $\mathcal{K}_{CBF}$ the robot will remain safe, and as long as it remains within $\mathcal{K}_{CLF}$ it will converge to its goal. However, there is no guarantee that these sets will have a non-empty intersection. We can easily construct ``bug-trap'' environment like that in Fig.~\ref{fig:two-modes-operation} where following the oCLF would lead the robot into an obstacle, violating the oCBF. In these situations, it is common to relax the oCLF condition, allowing the robot to remain safe even if it can no longer reach the goal \cite{Castaneda2020}. Although this relaxation preserves safety, it can lead to \textit{deadlock} when the robot becomes stuck and unable to reach its goal.

To avoid deadlock (i.e. ensure \textit{liveness}), we propose a hybrid controller with two modes, \textbf{G}oal-seeking and \textbf{E}xploratory, as in Fig.~\ref{fig:two-modes-operation}. We will first intuitively explain the operation of these two modes, then provide a formal description of each mode and a proof of the safety and liveness of this controller.

Both modes enforce the oCBF condition $u \in \mathcal{K}_{CBF}$ at all times, but they differ in their use of the oCLF. The goal seeking mode enforces the oCLF condition along with the oCBF condition, $u \in \mathcal{K}_{CBF} \cap \mathcal{K}_{CLF}$, maintaining safety while approaching the goal. When it is no longer feasible to satisfy both the oCLF and oCBF conditions, the robot switches to the exploratory mode, where it ignores the oCLF condition and executes a random walk around a level set of the oCBF (this can be seen as randomly exploring the region near the surface of the obstacle that caused it to become stuck, inspired by classic ``right-hand rules'' for maze navigation but adapted for a nonlinear dynamical system). The robot switches back to the goal-seeking mode when it encounters a region where the oCLF has decreased below the value where the robot became stuck (indicating that the robot is now closer to the goal than it was when beginning its exploration). This process is visualized in Fig.~\ref{fig:two-modes-operation}. Note that this controller is substantially different from the vision-based hybrid controller in \cite{anonymous2021robust}; ours switches between seeking the goal and exploring an unknown environment, while \cite{anonymous2021robust} switches between different controllers hand-designed for navigating a known environment.

\begin{figure}[th]
    \centering
    \includegraphics[width=0.9\linewidth]{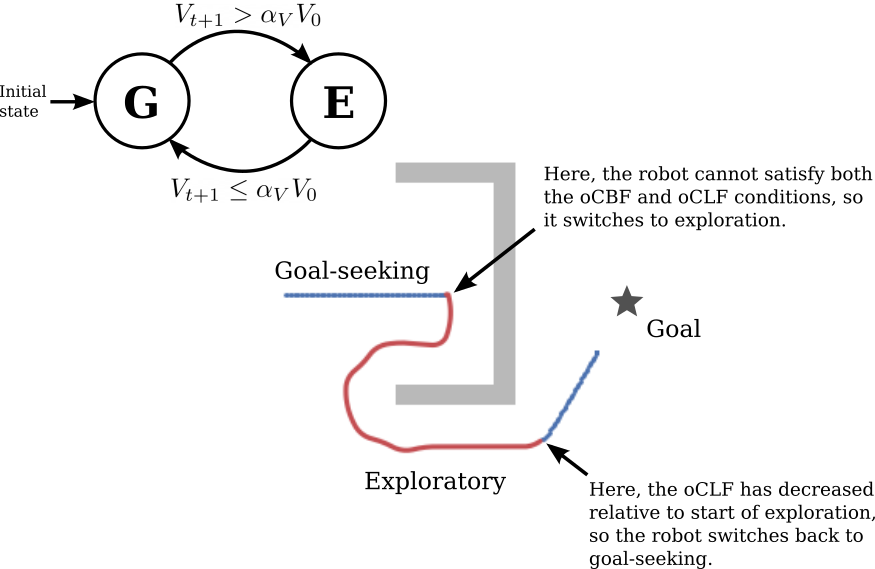}
    \caption{The hybrid control scheme allows the robot to escape potential deadlocks, as in this bug-trap environment.}
    \label{fig:two-modes-operation}
\end{figure}

Formally, the robot's behavior is given by solving an optimal control problem with a one-step horizon (in the goal-seeking mode) or executing a constrained stochastic policy (in the exploratory mode). As opposed to a traditional MPC policy, which must consider a multi-step horizon, the oCBF and oCLF encode all the information needed to ensure long-term safety and convergence into a single step. Crucially, the oCBF and oCLF rely only on local information: Lidar observations and range and bearing relative to the goal.

In the goal-seeking mode, the robot solves the optimization:
\vspace{-1em}
\begin{subequations}
\begin{align}
    u = \argmin_u & \quad \lambda_1 ||u|| \label{eq:g_controller} \\
    \rm{s.t.} & \quad V_{t+1} - \alpha_V V_t \leq 0 \label{eq:g_clf}\\
    & \quad h_{t+1} - \alpha_h h_t \leq 0 \label{eq:g_cbf}
\end{align}
\end{subequations}
Constraint~\eqref{eq:g_clf} ensures that the controller progresses towards the origin, while constraint~\eqref{eq:g_cbf} ensures that the controller remains safe. In practice, \textcolor{black}{we solve an unconstrained analogue of this problem using a penalty method}, incorporating the oCLF and oCBF constraints into the objective with a penalty coefficients $\lambda_2$ and $\lambda_3$, where $\lambda_1 = 0.01$, $\lambda_2 = 1.0$, and $\lambda_3 = 10^3$ such that $\lambda_3 \gg \lambda_1 \gg \lambda_2$. \textcolor{black}{We find that this choice of penalty coefficients results in good performance in practice by effectively filtering out any solutions that do not satisfy the constraints~\eqref{eq:g_clf} and~\eqref{eq:g_cbf}.} We also note that even if this approach fails to find a feasible solution, the controller can immediately detect this fault, raise an error, and switch to a fail-safe mode. The robot transitions from the goal-seeking mode to the exploratory mode when the oCLF constraint in \eqref{eq:g_controller} becomes infeasible, indicating that the robot is stuck and must explore its environment to find a new path to the goal.

Once in the exploratory mode, the robot follows a stochastic policy by discretizing the input space and sampling from the (appropriately normalized) distribution
\begin{equation}
\begin{aligned}
    \rm{Pr}(u) \propto \begin{cases}
    0 \qquad h_{t+1} - (1-\alpha_h)h_t \geq 0 \vee\ |h_{t+1} - h_0| \geq \epsilon_h \\
    1 / \rm{exp}\left( \lambda_1[h_{t+1} - (1-\alpha_h)h_t]_+\right. \qquad \text{otherwise}\\
    \quad + \left . \lambda_2 [|h_{t+1} - h_0| - \epsilon_h]_+ + \lambda_3 v^2  \right) 
    \end{cases}
\end{aligned}\label{eq:e_controller}
\end{equation}
where $h_0$ is the value of the oCBF recorded upon entering the exploratory mode, $\epsilon_h$ is the width of the region around the obstacle that the policy should explore, $v$ is the forward velocity implied by control input $u$ (included to encourage exploration), $[x]_+ = \max(x, 0.001 x)$ is the leaky rectified linear unit, and $\lambda_i$ are coefficients specifying the relative importance of each term. In our experiments, $\lambda_1 = \lambda_2 = 10^3$, and $\lambda_3 = -0.1$ to encourage faster exploration. The robot switches from the exploratory to the goal-seeking mode when it reaches a state where the oCLF value decreases below $\alpha_V V_0$, where $V_0$ is the value of the oCLF observed upon most recently switching to the exploratory mode. Both the optimization problem \eqref{eq:g_controller} and the stochastic policy \eqref{eq:e_controller} are evaluated by discretizing the input space and using the approximate one-step lookahead strategy discussed in Section~\ref{lookahead}; \textcolor{black}{in practice, both policies can be evaluated quickly by parallelizing the search over the input space.}

Since the oCBF constraint is enforced in both the goal-seeking and exploratory modes, the hybrid controller will be safe regardless of which mode it is in. As a remark, the discrete-time formulation of oCBFs allows the robot to tolerate changes in its environment, both continuous (moving an obstacle), and discontinuous (adding an obstacle). As long as the change does not flip the sign of $h$ (in essence, as long as the new obstacle position is far enough away that the robot does not instantly become unsafe), the oCBF controller will recover and avoid the obstacle. If the new obstacle makes the robot unsafe, the oCBF can immediately detect that fault.

It remains to prove that this controller will eventually reach the goal. We will do this by combining the following lemmas.

\begin{lemma}\label{lemma_1}
While in the goal-seeking mode, the robot converges to the goal.
\end{lemma}
\begin{proof}
This property will hold as long as constraint~\eqref{eq:g_clf} holds for the optimal solution of \eqref{eq:g_controller}. Since the robot transitions out of the goal-seeking mode when this constraint is violated, this condition must hold whenever the robot is in this mode.
\end{proof}
\vspace{-0.8em}
\begin{lemma}\label{lemma_2}
As long as there exists a path from the robot's current state to the goal such that the oCBF condition holds along that path (i.e. as long as the goal is reachable), the robot will eventually exit the exploratory mode, at which point the value of the oCLF will have decreased by a factor of $\alpha_V$ relative to its value entering the exploratory mode.
\end{lemma}
\begin{proof}
Let $x_0$ denote the point where the system enters the exploratory mode, and let $h_0$ (resp. $V_0$) denote the value of the oCBF (resp. oCLF) observed at $x_0$. In the absence of any oCBF, the robot would be able to continue from $x_0$ along a trajectory towards the goal $\tilde{x} = \tilde{x}_1, \tilde{x}_2, \ldots$ such that $V_{t+1} \leq \alpha_V V_t$ at each point. There are two cases: either the goal lies inside the $h_0$-superlevel set tangent to $x_0$ (the goal lies inside an obstacle represented by the oCBF) or it does not. In the former case, no safe path exists to the goal (violating the assumption of this lemma). In the later case, since we assume contours of the oCBF to be closed, $\tilde{x}$ will eventually leave the $h_0$-superlevel set bordered by $x_0$. As a result, there exists a point along this hypothetical oCLF-only trajectory $\tilde{x}_T$ that lies near the $h_0$-level set of the oCBF and $V(\tilde{x}_T) \leq \alpha_V V(x_0)$. Since the system is assumed to be controllable and contours of the oCBF closed, the stochastic policy in \eqref{eq:e_controller} will eventually explore all states in the region around the oCBF $h_0$-level set. Thus, the stochastic exploration policy will eventually approach $\tilde{x}_T$ (after circling around the $h_0$-level set), at which point the hybrid controller will switch to the goal-seeking mode.
\end{proof}
\begin{theorem}
The hybrid goal-seeking/exploration controller in Fig.~\ref{fig:two-modes-operation} will eventually reach the goal.
\end{theorem}
\begin{proof}
Denote by $[t_i, T_i]$ the period of time the robot spends in the goal-reaching mode on the $i$-th time it enters that mode, and consider the mode-switching behavior of the robot over an infinitely long horizon. Since Lemma~\ref{lemma_2} shows that the system will always exit the exploratory mode, either there are infinitely many episodes in which the robot is in the goal-seeking mode (i.e. $i = 0, 1, \ldots, \infty$), or the robot eventually reaches the goal-seeking mode and stays there (i.e. $i = 0, 1, \ldots, N_g$ and $T_{N_g} = \infty$). We present the proof for the first case; the second case follows similarly. Consider the trace of oCLF values
\begin{equation}
    V({t_0}), \ldots, V({T_0}), V({t_1}), \ldots, V({T_i}), \ldots, V({t_{i+1}}), \ldots \label{eq:V_seq}
\end{equation}
and let $t'$ denote an index into this sequence (i.e. the cumulative time spent in the goal-reaching mode). We will show that this sequence is upper-bounded by the geometric sequence $\alpha_V^{t'} V({t_0})$. Within any episode $[t_i, T_i]$, the sequence $V({t_i}), \ldots, V({T_i})$ is upper-bounded by $V({t}) \leq \alpha_V^{(t - t_i)} V({t_i})$ by the properties of the oCLF. Further, by Lemma~\ref{lemma_2} we have the relation $V(t_{i+1}) \leq \alpha_i V(T_{i})$. As a result, we see that each term of the concatenated sequence~\eqref{eq:V_seq} obeys $V(t' + 1) \leq \alpha_V^{t'}V(t')$. As a result, as $t' \to \infty$, $V(t') \to 0$. By the properties of the oCLF, this implies that the robot's state approaches the goal as the cumulative time spent in the goal-reaching state increases (regardless of intervening mode switches). Since Lemma~\ref{lemma_2} implies that cumulative goal-reaching time goes to infinity with cumulative overall time (and our system operates in discrete time), this completes the proof that our hybrid controller converges to the goal.
\end{proof}

\subsection{Approximate one-step lookahead}\label{lookahead}

\begin{figure}[h]
    \centering
    \includegraphics[width=\linewidth]{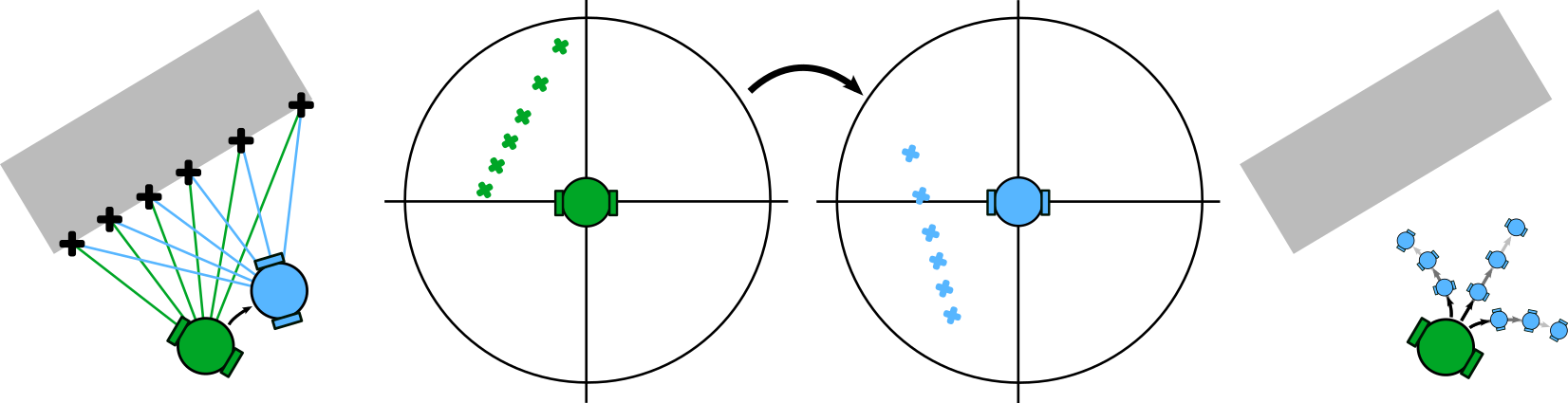}
    \caption{Future observations can be approximated by translating and rotating past observations. By considering multiple possible futures, the robot can select an action according to one-step lookahead policies like~\eqref{eq:g_controller} and \eqref{eq:e_controller}.}
    \label{fig:lidar-lookahead}
\end{figure}

The control policies~\eqref{eq:g_controller} and \eqref{eq:e_controller} require a prediction of $h_{t+1}$ and $V_{t+1}$ after executing a possible control action $u$. In turn, predicting $h_{t+1}$ and $V_{t+1}$ requires predicting the future observations $o_{t+1}$ and future range and bearing $(\rho_{t+1}, \phi_{t+1})$ to the goal.
To generate these predictions, we use prior knowledge of the robot's dynamics to predict\footnote{For robots with local dynamics (as in Section~\ref{assumptions_dynamics}), this can be done without a state estimate; otherwise, a state estimate is required.} $\Delta x = x_{t+1} - x_t$. This change in state implies a change in the robot's position and orientation in the 2D workspace, defining a transformation $T_u \in SE(2)$, describing the future position in the robot's current local frame. This transformation can be used to predict future observations, subject to errors introduced by discontinuities in the environment:
\begin{align}
    o^i_{t+1} &= T_u^{-1} o^i_{t} &\textit{(Lidar update)}\label{lidar_update}\\
    g_{t+1} &= T_u^{-1} \rho_{t} \mat{\cos\phi & \sin\phi}^T &\textit{(Goal update)}\label{goal_update}\\
    \rho_{t+1} &= ||g_{t+1}||;\ \phi_{t+1} = \angle g_{t+1} &\textit{(Range and bearing)}\label{range_update}
\end{align}
The predicted $o_{t+1}$, $\rho_{t+1}$, and $\phi_{t+1}$ can be used to evaluate $h_{t+1}$ and $V_{t+1}$ in the goal-seeking and exploratory control policies. Of course, the approximate nature of this update means that the safety and convergence properties of the oCBF and oCLF are no longer guaranteed to apply, but this issue can be mitigated in practice by adding a margin to the oCBF/oCLF conditions (e.g. $V_{t+1} - (1-\alpha_V) V_t \leq -\gamma_V$ in~\eqref{eq:g_controller} with $\gamma_V \geq 0$, as in \cite{Dean2020a}). Following our discussion of how to learn an oCBF and oCLF in the next section, we will use empirical results from simulation and hardware to demonstrate that this hybrid lookahead controller exhibits safe and convergent behavior in a range of environments, despite its use of approximations.

\section{Learning Observation-space CBFs}\label{learning}

Recent works have successfully applied neural networks to learn state-space CBFs \cite{Qin2021,Peruffo2020} and CLFs \cite{Richards2018,Abate2020,Chang2019,Chang2021}. These approaches sample points from the state space, label those points as either safe or unsafe, and then train a neural network to minimize the violation of the state-dependent equivalents of conditions~\eqref{ocbf_safe}--\eqref{ocbf_dynamics} (for CBFs) and~\eqref{oclf_pd}--\eqref{oclf_dynamics} (for CLFs) on all points on the training set. To evaluate conditions~\eqref{ocbf_dynamics} and~\eqref{oclf_dynamics}, these approaches either use a fixed control policy~\cite{Richards2018,Abate2020,Peruffo2020} or learn a single control policy~\cite{Qin2021,Chang2019,Chang2021}. These approaches are useful when the safe and unsafe regions of the state space are well-defined and static, but they essentially memorize which regions of the state space are safe and which are unsafe. As a result, they do not generalize when the safe/unsafe regions change, for example when moving between environments with different arrangements of obstacles.

There are two key insights in our approach to learning oCBFs and oCLFs, as contrasted with prior approaches to learning CBFs and CLFs. First, by defining these certificates in terms of observations allows the learned certificates -- especially the oCBF -- to generalize much more easily to new environments. Second, we train our oCBF and oCLF to not simply to satisfy conditions~\eqref{ocbf_dynamics} and~\eqref{oclf_dynamics} for a single example controller but to ensure the feasibility of the exploratory control policy~\eqref{eq:g_controller}. This task-specific learning approach allows us to learn an oCLF and oCBF that are directly relevant to the task at hand, as well as improve generalization by not committing to a single control policy.

In our approach, the oCLF is defined as $V(\rho, \sin\phi, \cos\phi) = V_\omega(\rho, \sin\phi, \cos\phi) + \rho^2 + (1-\cos\phi)/2$, where $V_\omega$ is a neural network with parameters $\omega$ (two hidden layers, 48 ReLU units each, one output dimension), $\rho$ and $\phi$ are the range and heading to the goal in the local frame, respectively, and the last term encodes the prior belief that the oCLF should correlate with distance from the goal.

Because the Lidar observations $o$ are symmetric under permutation, we must adopt a different structure for learning an oCBF that reflects this symmetry. Taking inspiration from \cite{Qin2021}, we define a permutation-invariant encoder $e(o) = \max_i e_\theta(o^i)$, which passes every point in the Lidar observation through an neural network $e_\theta$ with parameters $\theta$ (two hidden layers, 48 ReLU units each, 48 output dimensions). $e_\theta$ lifts each two-dimensional point to a 48-dimensional space, and we then take the element-wise maximum among all observed points. We then construct the oCBF as $h(o) = h_\sigma(e(o)) - \min_i ||o^i|| + d_c$, where $h_\sigma$ is a neural network with parameters $\sigma$ (two hidden layers, 48 ReLU units each, one output dimension) and the minimum is taken over all points in the Lidar observation $o$. The final two distance terms impose the prior belief that the oCBF should correlate with the distance to the nearest obstacle, and we use an additional learned term to ensure that $h$ satisfies~\eqref{ocbf_dynamics}.

To train these neural networks, we sample $N$ points uniformly from the state space, use a simulated Lidar model to compute the observation $o$ at each point, then use stochastic gradient descent to minimize the loss $L = \frac{1}{N}\sum_{i=0}^N (a_1 \text{ReLU}(\epsilon_h + h(o_i))\mathbf{1}_{safe}$ $+ a_2 \text{ReLU}(\epsilon_h - h(o_i))\mathbf{1}_{unsafe} + a_3 L_g  )$, where $a_1 = a_2 = 100$, $a_3 = 1$, $\mathbf{1}_{safe}$ and $\mathbf{1}_{unsafe}$ are indicator functions for the safe and unsafe sets, and $L_g(o)$ is the optimal cost of the relaxed goal-seeking control policy~\eqref{eq:g_controller}. We also find that adding a regularization term for the L2-norm of $h_\sigma$ and $V_\omega$ helps improve training performance. Loss curves for the training and validation datasets are shown in Fig.~\ref{fig:loss_curves}, covering 72 training epochs.
\begin{figure}[h]
    \centering
    \includegraphics[width=0.6\linewidth]{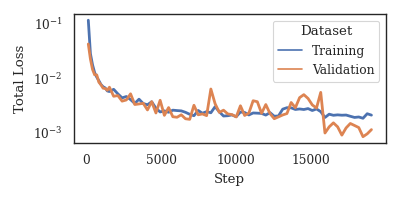}
    \caption{Loss $L$ on a dataset of 10,000 points ($10\%$ reserved for validation).}
    \label{fig:loss_curves}
\end{figure}

\subsection{Verification of learned certificates}

\textcolor{black}{
In Section~\ref{hybrid_controller}, we prove that if $h$ and $V$ are valid oCBFs and oCLFs, then our hybrid controller is guaranteed to reach the goal and remain safe. However, because we learn the oCBF and oCLF using neural networks trained on a finite number of sampled points, we do not have an explicit guarantee that the learned $h$ and $V$ will be valid. Related works in certificate learning rely on probabilistic verification methods such as those discussed in \cite{Qin2021} and \cite{Boffi2020}. These methods draw on results from statistical learning theory to show that if a learned certificate is valid (with some margin) on a large set of sample points, then there is a high probability that it will be valid in general. Even if a certificate is invalid in a small subset of the state space, almost-Lyapunov theory \cite{Liu2020} suggests that global safety and stability guarantees may still exist.}

\textcolor{black}{To check the validity of our learned certificates, we sample 100,000 points from the state space and check whether constraint~\eqref{eq:g_cbf} is feasible at each point; we found that this constraint is feasible at each tested point. However, we acknowledge that this sampling analysis falls short of a formal proof, and factors such as sensor noise and inaccuracies in the approximate lookahead model mean that the formal guarantees from Section~\ref{hybrid_controller} may not hold in practice. To demonstrate empirically that our learned certificates yield a safe, performant controller in practice, the next section presents experimental results showing that our controller remains safe and reaches the goal in a wide range of simulated and hardware environments. In our future work, we hope to extend our theoretical analysis to explain this strong empirical performance.
}





\section{Experiments}\label{experiments}

We validate our hybrid lookahead controller in two ways. First, we compare its performance against model-predictive control (MPC), reinforcement learning (RL), and state-based CBF/CLF baselines in simulation with a Dubins car model, demonstrating that our controller exhibits improved performance and generalizes well beyond the training environment. Second, we deploy our controller in hardware to navigate a mobile robot through a changing environment.

\subsection{Generalization beyond training environment}

To assess controllers' ability to generalize beyond the environments in which they were trained, both our oCBF/oCLF controller, an \textcolor{black}{end-to-end} RL controller, and a state-based CBF/CLF controller were trained in a static environment with 8 randomly placed obstacles. For the RL agent, we use \textcolor{black}{a neural-network policy with the same inputs as our controller, trained in an end-to-end manner using} proximal policy optimization (PPO) as implemented in the OpenAI Safe RL benchmarks~\cite{Ray2019}. The RL agent was trained with a large reward for reaching the goal, a smaller dense reward equal to $-\rho^2 + (1-\cos\phi)/2$, and a large cost for colliding with objects. The state-based CBF/CLF method was trained identically to our own method except with $h(x) = h_\sigma(x)$. The MPC method is model-based and does not need training; it constructs a convex approximation of the locally-observable free space (inspired by \cite{Deits2015}), plans a path towards the goal within that region, and executes the first \SI{0.1}{s} of that path. \textcolor{black}{When it becomes stuck, MPC attempts to navigate around the boundaries of obstacles.} All controllers were run at \SI{10}{Hz}, and numerical simulations of continuous-time Dubins car dynamics occurred at \SI{100}{Hz}.

Example trajectories for each controller navigating a Dubins car through a randomized 2D environment are shown in Fig.~\ref{fig:random_environment}, and Fig.~\ref{fig:generalization_bar_chart} shows the collision and goal-reaching rates for each controller across 500 such randomized environments (along with the average time required to reach the goal when successful). These data show that our proposed approach achieves a $100\%$ safety rate while reaching the goal within \SI{10}{s} in $93.2\%$ of trials (less than $100\%$ due to obstacles occluding the goal in some cases). Our method significantly outperforms the other learning-based methods (both of which have lower safety and goal-reaching rates due to difficulty generalizing beyond the training environment). MPC achieves slightly worse safety and goal-reaching rates than our approach: \textcolor{black}{$99.6\%$ and $90.4\%$, respectively}. Typical behaviors for each controller are shown in Fig.~\ref{fig:random_environment}: the state-based learned CBF does not adapt to the new environment and is unsafe, the PPO policy is not precise enough to pass through the gap to reach the goal (although it remains safe). Only our approach and MPC consistently reach the goal safely, \textcolor{black}{although our approach achieves a slightly higher goal-reaching rate and is much less computationally intensive (\SI{18}{ms} of online computational time for our method vs. \SI{102}{ms} for MPC).}

\begin{figure}[h]
    \centering
    \includegraphics[width=0.8\linewidth]{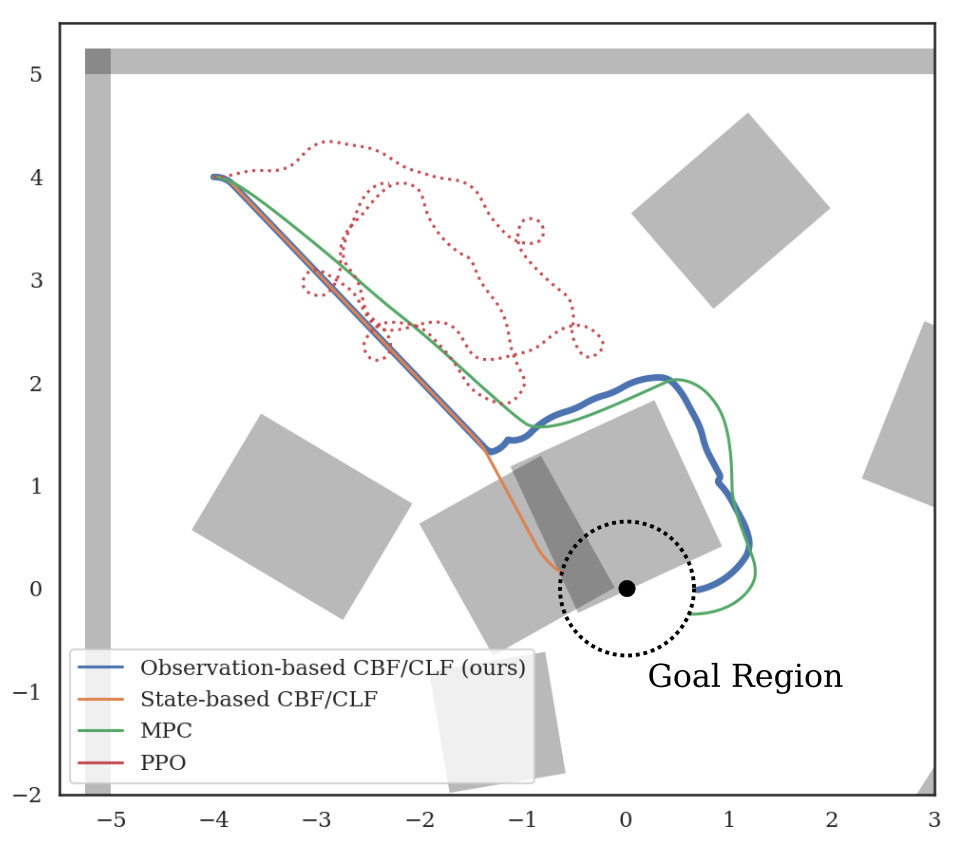}
    \caption{Plot of trajectories for observation-space CBF/CLF, state-based CBF/CLF, MPC, and end-to-end RL policies in random environment.}
    \label{fig:random_environment}
\end{figure}
\begin{figure}[h]
    \centering
    \includegraphics[width=0.8\linewidth]{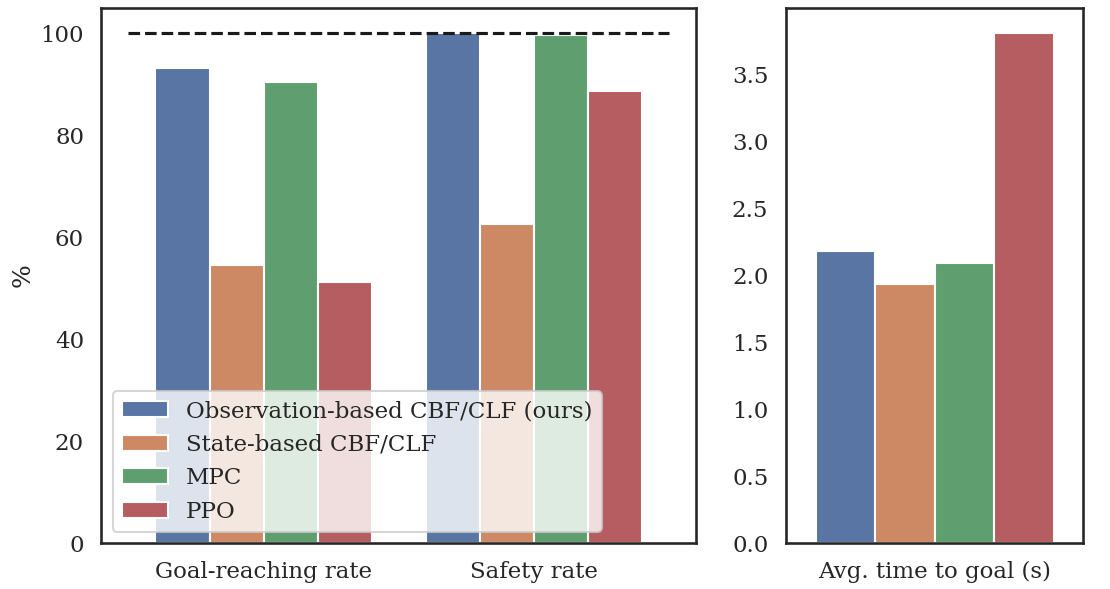}
    \caption{Goal-reaching rate, safety rate, and time-to-goal for oCBF/oCLF, state-based CBF/CLF, MPC, and PPO on waypoint tracking, averaged over 500 randomized environments. Trial runs were limited to a maximum of \SI{10}{s}, and time-to-goal was averaged only over trials that reached the goal.}
    \label{fig:generalization_bar_chart}
\end{figure}

\subsection{Hardware demonstration}\label{hardware}

To validate our controller in the real world, we deployed our controller on a TurtleBot 3 mobile robot. We model this platform as a Dubins car with control inputs for forward and angular velocity. The TurtleBot is equipped with an HLS-LFCD2 2D laser scanner with $1^\circ$ resolution in a $360^\circ$ range, which we downsample to 32 equally-spaced measurements. Lidar scans are preprocessed to convert from polar to local Cartesian coordinates before being passed to the learned encoder. Range and bearing to the goal are estimated using odometry, demonstrating robustness to imperfect state estimation. The controller is implemented with a zero-order hold at \SI{10}{Hz}, but Lidar data are only available at \SI{5}{Hz} (stale scans were used when new data were not available). \textcolor{black}{Although we do not explicitly model sensor uncertainty, we observed both small-magnitude noise and spurious detections in our Lidar data (these effects can be seen in the supplementary video), demonstrating that our controller can handle the noise resulting from real sensors.} We compare our method against the TurtleBot's built-in SLAM module in combination with an online dynamic window path planner.

Fig.~\ref{fig:hw} shows our controller successfully avoiding both fixed obstacles and an obstacle thrown into its path mid-experiment. Video footage of this experiments (and others in different environments) is included in the supplementary materials. Our controller successfully escapes the initial trap and avoids the thrown obstacle to reach the goal, while the combined SLAM/planning system becomes stuck when the new obstacle is added. The planner eventually becomes unstuck (after rotating the robot in place to re-create the map), but this issue demonstrates the difficulty of guaranteeing liveness for an online SLAM and planning system. In contrast, our controller is entirely reactive, allowing it to avoid the new obstacle without pausing, providing a significant advantage compared to model-based online planning methods.

\begin{figure}[ht]
    \centering
    \includegraphics[width=0.8\linewidth]{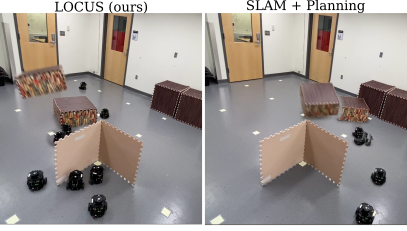}
    \caption{(Left) Our controller navigating around static and dynamic obstacles, shown as a composite image. (Right) A SLAM module paired with an online global planner gets stuck when an unexpected object is added to the scene. A video of these experiments is included in the supplementary materials.}
    \label{fig:hw}
\end{figure}

\section{Conclusions and Future Work}

Safely integrating rich sensor models into autonomous control systems is a challenging problem at the intersection of robotics, control theory, and machine learning. In this work, we explore one direction for constructing safe, generalizable perception-based controllers. We use neural networks to learn an observation-based control barrier function, which we combine with a hybrid control architecture that allows us to prove the safety and liveness of the resulting learning-enabled controller. We demonstrate in simulation that this architecture outperforms other learning- and model-based control systems \textcolor{black}{(including end-to-end learning using RL)}, and we successfully deploy our controller in hardware.

In addition to these successes, it is important to highlight the drawbacks of our approach, which point to interesting areas of future work. In particular, we rely on an approximate model of the Lidar sensor to predict observations one step into the future. This reliance prevents us from easily extending our approach from Lidar to image data, as it is much more difficult to construct an approximate model for image updates. In future work, we hope to replace this model-based update with a generative model learned from data. Using a learned, generative update might also allow us to relax the assumption of local dynamics, presenting a promising line of future study.

\bibliographystyle{IEEEtran}
\bibliography{IEEEabrv,main}

\end{document}